\documentclass[11pt]{article}
\usepackage[utf8]{inputenc}
\usepackage{amsmath, amssymb, amsthm, color, hyperref}
\usepackage{algorithm2e}
\usepackage{graphicx}
\newcommand{\omitthis}[1]{}
\newcommand{\E}{{\bf E}\,}

\newtheorem{example}{Example}

\newtheorem{proposition}{Proposition}
\newtheorem{remark}{Remark}

\newcommand{\comment}[1]{}

\def\mAth{\mathsurround=0pt}
\def\eqalign#1{\,\vcenter{\openup1\jot \mAth
  \ialign{\strut\hfil$\displaystyle{##}$&$\displaystyle{{}##}$\hfil
    \crcr#1\crcr}}\,}

\newcommand{\half}{\mbox{$\frac{1}{2}$}}
\newcommand{\quarter}{\mbox{$\frac{1}{4}$}}
\newcommand{\eq}{\begin{equation}}
\newcommand{\eeq}{\end{equation}}

\long\def\omitthis#1{}
\newcommand{\ist}{{\rm({\it i}\,{\rm )}}}
\newcommand{\ind}{{\rm({\it ii}\,{\rm )}}}
\newcommand{\ird}{{\rm({\it iii}\,{\rm )}}}

\newcommand{\itemi}{\item[\ist]}
\newcommand{\itemii}{\item[\ind]}
\newcommand{\itemiii}{\item[\ird]}

\makeatletter
\DeclareRobustCommand{\nand}{\mathbin{\mathpalette\n@and@or\land}}
\DeclareRobustCommand{\nor}{\mathbin{\mathpalette\n@and@or\lor}}

\newcommand{\n@and@or}[2]{%
  \vphantom{#2}%
  \ooalign{$\m@th#1#2$\cr\hidewidth$\m@th#1\sim$\hidewidth\cr}%
}

\makeatother

\begin{document}
\title{\bf Remarks on Utility in Repeated Bets}
\author{Nimrod Megiddo\thanks{IBM Almaden Research Center, San Jose, California}}
\date{June 2023}
\maketitle
\parindent=0pt

\begin{abstract}
The use of von Neumann -- Morgenstern utility is examined in the context of multiple choices between lotteries. 
Different conclusions are reached if the choices are simultaneous or sequential. It is demonstrated that utility cannot be additive.
\end{abstract}

\section{\hskip -15pt. Introduction}
We examine here the issue of determining utility values in the context of repeated bets with monetary prizes. 
We assume the axioms of utility of von Neumann and Morgenstern (vNM)  \cite{von1947theory}, \
which imply the principle of expected utility maximization.  
Thus, we do not present a criticism of these axioms as offered by the so-called Allais paradox \cite{allais1953comportement}.  
We consider a multiple-decision problem, where the monetary rewards are nonnegative.
In the case of sequential decisions, due to the nonnegativity, there is no issue of optimal stopping.
In many theoretical multi-stage models, for example in repeated games \cite{mertens1990repeated}, the standard assumption is that utilities from the single stages add up, possibly subject to a discount factor. 
We show that this additivity assumption is inconsistent with the application of the vNM theory.

\section{\hskip -15pt. The basic bet} \label{sec:basic}
As a motivating example, we consider a decision problem of choosing between two rewards:
Reward A is a cash prize of $\$400$, and
Reward B is a lottery, denoted $[\half(\$1000), \half(0)]$, where the sum of $\$1000$ is awarded with probability $\half$, and zero otherwise.

The vNM axioms imply the following:
\begin{enumerate}
\itemi 
there exists a number $\alpha$, $0 <\alpha < 1$,
such that the decision maker (DM) is indifferent between Reward A and a lottery,
$[\alpha(\$1000),(1-\alpha)(0)]$
 that gives \$1000 with probability $\alpha$
and zero otherwise;
\itemii
 there exists a utility function $u$ whose domain is the set of probability distributions\\
$(x_0, \,  x_{400},\,  x_{1000})\ge 0$  ($x_0+x_{400}+x_{1000}=1$)
over 
the set $\{0,\$400,\$1000\}$,  such that 
$u(\$1000) \equiv u(0,0,1) =1$, $u(0)\equiv u(1,0,0)= 0$ and $u(\$400) \equiv u(0,1,0) = \alpha$; and
\itemiii
$u(x_0, 0, x_{1000}) = x_{1000} $. 
\end{enumerate}
It follows that Reward B has utility $\half$, and the DM would choose Reward A if $\alpha > \half$,
and only if  $\alpha \ge \half$.

\section{\hskip -15pt. Multiple simultaneous bets}
It is important to note that the vNM utility theory is applicable only in a single decision or game problem.
The theory includes the axiom of  ``Reduction of Compound Lotteries'' axiom, which implies that a multi-stage
problem can be reduced to a single stage. 

Consider repeating the bet of Section \ref{sec:basic} twice.  
In this case the DM may either 
choose Reward A twice and receive $\$800$, 
choose Reward A once and Reward B once and receive $[ \half(\$1400), \half( \$400)]$
or choose Reward B twice and receive $[\quarter (\$2,000), \, \half (\$1000), \, \quarter (0)]$.

\begin{remark}  \label{rem:1} \rm
If we used the utility function of the single stage and added the utilities of the two stages, then we would get values as follows.
Choosing A twice yields $2\alpha$, choosing $A$ once and $B$ once yields $\alpha + \half$, and choosing $B$ twice
yields $1$. 
Obviously, such a valuation would lead to a conclusion that the DM would choose Reward A in the single stage if
and only if he would choose Reward A in each of the stages of the two-stage problem.
\end{remark}

In general, in a (stochastically-independent) repetition of the basic bet $n$ times, if the DM chooses Reward A in $k$ times ($0\le k\le n$) 
and Reward B in the other $n-k$ times,
then adding up the vNM utility values of the single bets yields the utility value of
$ \alpha \, k +  \half (n-k) $.
Hence, the ``optimal'' value of $k$ is equal to $0$ if $\alpha > \half$ and equal to $n$ if $\alpha < \half$.  
The vNM axioms  do not imply such additivity though. 
We first consider the case where all of the $n$ bets take place simultaneously or, equivalently, 
where the DM must decide in advance which reward is chosen in each stage.

The consequence of additivity of utility described above cannot be true in general because of the laws of large numbers.  
Denote by $B_n$ the random variable of the total cash reward if Reward B is chosen at every stage.
The expectation of $B_n$ is equal to $\$500\,n$ and the standard deviation is $\$250\sqrt{n}$. 
The probability that $B_n > \$400\,n $ tends to $1$ as $n$ tends to infinity, so there must exist an $N$ such that
for every $n\ge N$, the DM would prefer to choose Reward B rather than Reward A in each bet of  the $n$-bet problem. 

If the $n$ bets take place simultaneously, then a single vNM utility function $u_n(\cdot)$ can be formulated for this situation as follows.
The possible cash rewards have the form $c_{k\ell} = \$400\,k +\$1000\,\ell$, where $k,\ell\ge 0$ and $k+\ell \le n$.
First, set $u_n(\$1000\,n) = 1$ and $u(\$0)= 0$.
Next, let $u_n(c_{k\ell})$ be equal to the probability $p_{k\ell}$ such that the DM is indifferent between the cash reward $c_{k\ell}$ and
the lottery $[u_n(c_{k\ell})(\$1000\,n),\,  (1-u_n(c_{k\ell}) (0) ]$, 
where the reward of $\$1000\,n$ is received with probability $u_n(c_{k\ell})$ and zero otherwise.
Finally, choosing Reward A in $k$ of the bets and Reward B in the other $r = n-k$ bets, 
which we denote by $(kA, rB)$, yields the cash reward $c_{k\ell}$ with probability
$2^{-r}{r \choose \ell}$ ($\ell=0\ldots,r$), hence the utility is equal to
\[ u_n( kA, rB) \equiv  2^{-r} \sum_{\ell=0}^{r}{r \choose \ell} u_n(c_{k\ell}) ~.\]

\paragraph{Approximation using a utility-of-money function.}  
A simplification of the above analysis can be made if a specific utility-of-money function $\phi_n:[0,\, 1000\,n] \to [0,1]$
can be assumed to  satisfy the expected utility property, i.e., the utility from a lottery that gives $c_j$ with probability $p_j$ is
equal to $\sum_j p_j \phi_n(c_j)$.
For example,
\[ \phi_n(x) = \frac{\log(1+x)}{\log(1+ 1000\,n) }~.\]
In this case,
\[ u_n( kA, rB) = 2^{-r} \sum_{\ell=0}^{r}{r \choose \ell} \phi_n(c_{k\ell}) ~.\]
We can use the approximation of the binomial distribution by a normal distribution with mean $\frac{r}{2}$ and variance $\frac{r}{4}$, i.e.,
replace $ 2^{-r}{r\choose \ell}$ by 
\[  \sqrt{\frac{2}{\pi r}} \cdot  \exp\left( -\frac{ 2(\ell - \frac{r}{2})^2}{r} \right) 
=   \sqrt{\frac{2}{\pi r}} \cdot  \exp 
%  \left(-\frac{ 2(\ell - \frac{r}{2})^2}{r} \right) 
\left(  - \frac{2\ell^2}{r} + 2\ell - \frac{r}{2} \right)
~.\]
For a large $r$, the concentration around the mean $\frac{r}{2}$ implies the following approximate utility:
\[ u_n( kA, rB) \approx  \frac{\log(1 + 400\,k + 500\, r)}{\log(1+ 1000\,n) }~.\]
Obviously, for a large $n$, the optimal pair is  $(k, r) = (0,n)$.

\section{\hskip -15pt. Multiple sequential bets}
If the bets are made sequentially rather than simultaneously, then the DM does not commit in advance to a certain
number of times of choosing the different rewards even though the total number of stages is known in advance.
In general, the total amount of cash received during stages $1\ldots,n-1$ may
affect the preferences of the DM with regards to outcomes of stage $n$.
For the sake of the present discussion, suppose it does not. 
If so, then at stage $n$ the DM is faced with a decision problem equivalent to the one-stage problem.
Thus, in the last stage the DM may prefer to receive Reward A~$\equiv \$400$ rather than Reward B~$\equiv [\half(\$1000 ),\, \half(0)]$.
Thus, at stage $n-1$, the choice is between Reward A*~$\equiv \$800$ and
Reward B*~$\equiv [\half (\$1400),\, \half (\$400)]$.
%  betting on receiving \$1400 or \$400 each with probability $\half$.
It is conceivable that although the DM prefers A to B, they would still prefer Reward B* to Reward A*.
\begin{proposition}
Let $u: [0,M] \to [0,1] $ be a monotone increasing utility-for-cash function 
Let $a, b, c$ be such that $0 \le a < c < b < M$. 
For every $x \in [-a, M-b]$, let $\alpha(x)$ be the value that satisfies
\[  u(c+x) = \alpha(x) \cdot u(b+x) + (1-\alpha(x)) \cdot u(a+x) \] 
i.e.,
\[ \alpha(x) =  \frac{u(c+x) - u(a+x)}  {u(b+x) - u(a+x)}~.  \]
Let $I=[\beta,\delta] \subseteq [0,M]$ be any interval and
let $I^* = [a+\gamma, b+\delta]$.
Under these conditions, 
\begin{enumerate} 
\itemi
if $u$ is strictly concave over $I^*$, then $\alpha(x)$ is monotone decreasing over $I$,
\itemii
if $u$ is strictly convex over $I^*$, then $\alpha(x)$ is monotone increasing over $I$, and
\itemiii
if $u$ is linear over $I^*$, then $\alpha(x)$ is constant over $I$.
\end{enumerate}
\end{proposition}
\begin{proof}
We have
\[ \alpha'(x) =  \frac{ (u(b+x) - u(a+x))  (u'(c+x) - u'(a+x))  -   (u'(b+x) - u'(a+x))  (u(c+x) - u(a+x))}{ (u(b+x) - u(a+x))^2}~,\]
and for all $x$
\[ \eqalign{
u(b+x) - u(a+x) >&\  0~,\cr
 ~\mbox{ and } ~ u(c+x) - u(a+x) <&\  0~.\cr
 }\]
 \ist\
If $u$ is strictly concave over $I^*$, then for every $x \in I$,
\[ \eqalign{
u'(c+x) - u'(a+x) <&\  0~,\cr
 ~\mbox{ and } ~  u'(b+x) - u'(a+x) <&\ 0  \cr
 }\]
and  it follows that $\alpha'(x) < 0$.\\
\ind\
If $u$ is strictly convex over $I^*$, then for every $x \in I$,
\[ \eqalign{
u'(c+x) - u'(a+x) >&\  0~,\cr
 ~\mbox{ and } ~  u'(b+x) - u'(a+x) >&\ 0  \cr
 }\]
and  it follows that $\alpha'(x) >0$.\\
\ird\
If $u$ is linear over $I^*$, then for every $x \in I$,
\[ \eqalign{
u'(c+x) - u'(a+x) = &\  0~,\cr
 ~\mbox{ and } ~  u'(b+x) - u'(a+x) = &\ 0  \cr
 }\]
and  it follows that $\alpha'(x) =0$.
\end{proof}

Denote by $U$ the extension of the utility-for-cash $u$ to random variables, i.e., $U(X) = \E[u(X)]$.

Suppose $u$ is monotone increasing and strictly concave.
We have
\[ \alpha(x) = \frac{400 \cdot u'(x + \theta_1(x) \cdot 400 ) }{ 1000 \cdot u'(x +\theta_2(x) \cdot 1000)} ~ \]
for some $\theta_1(x), \theta_2(x) \in [0,1]$.
Since $\alpha'(x) > 0$ and decreasing, it must have a limit $\alpha^*$ as $x$ tends to infinity.
If $\alpha^* > 0$ then the following limit exists:
\[ \lim_{x \to\infty} \alpha(x) = \frac{400\, \alpha^*}{1000\,\alpha^*}  = 0.4 ~.\]
Otherwise,  if $\alpha^*=0$ , then for every $0< \beta< 1$,
\[ \lim_{x\to\infty} (u(400+x) -  \beta\cdot u(x+1000) - (1-\beta)\cdot u(x)  = 0 ~.\]

The extension of random variable $D$ follows from the fact that if $D$ takes the
values $d_1,\ldots,d_r$ with respective probabilities $p_1,\ldots,p_r$, then
\[ u(c+D) = \sum_i p_i \cdot u(c+d_i) ~.\]

\subsubsection*{Dynamic-programming solution}

The problem of sequential bets can be analyzed recursively as follows.
First,  we introduce a concept of conditional utility. 	
Consider cash rewards $r_1$ and $r_2$ received in two installments, 
and let $u$ be the utility-for-cash function. 
The utility from the total reward is $u(r_1+r_2)$. 
However, after receiving $r_1$, the utility from the second reward is not $u(r_2)$ 
(a quantity which is independent of $r_1$) 
but rather
\[  U(r_2\,|\,r_1) \equiv u(r_1+r_2) - u(r_1) ~,\]
which we call the {\em conditional utility} of $r_2$, given the previous reward $r_1$.

Note that if the utility-for-cash $u$ is linear, $u(x) = \xi\cdot x$, then  
$U(r_2\,|\,r_1) = \xi\cdot(r_1+r_2) - \xi\cdot r_1 = u(r_2)$.
This means that the utility of the second installment is independent of the first installment.
It follows that when the utility-for-cash is linear, the decision in every stage is as if this is a one-stage problem.

\begin{example}  \label{ex:3}\rm
Consider a 2-stage problem where at each stage the DM has to choose between
$A\equiv \$400$ and $B\equiv  [\half (\$1000),\, \half (0)]$.
Suppose the utility-for-cash function is $u(x) = \ln(1+x)$. 
Thus,
$u(0)=0$,  $u(\$400) = 5.994$ and $u(\$1000) = 6.909$,
and we have
\[  u(\$400) > \half u(\$1000) + \half u(0) ~.\]
Now,
$u(\$800)=6.686$,  $u(\$1400) = 7.245$ and $u(\$2000) = 7.601$. 
The preference in Stage 2 depends on the reward in Stage 1. 
If $r_1=\$400$, then
\[ U(0\,|\,\$400) =0~,\]
\[ U(\$400\,|\,\$400) = u(\$800) - u(\$400) =0.692 ~,\]
\[ U(\$1000\,|\,\$400) = u(\$1400) - u(\$400) =1.251~\]
and 
\[  U\$400\,|\, \$400) > \half U(\$1000 \,|\, \$400) + \half u(0 \,|\, \$400) ~.\]
Therefore, if the DM chooses $A$ in Stage 1, then he chooses the same in Stage 2, and the total reward
is $\$800$ with utility $6.686$.
If $r_1=\$1000$, then
\[ U(0\,|\,\$1000) =0~,\]
\[ U(\$400\,|\,\$1000) = u(\$1400) - u(\$1000) =0.336 ~,\]
\[ U(\$1000\,|\,\$1000) = u(\$2000) - u(\$1000) =0.693~\]
and 
\[  U(\$400\,|\, \$1000) <  \half U(\$1000 \,|\, \$1000) + \half u(0 \,|\, \$1000) ~.\]
Therefore, if the DM chooses $B$ in Stage 1, then in case the reward is zero, he would choose $A$ in Stage 2, and if the reward is $\$1000$ , then he would choose $B$ again in Stage 2.  
The utility in this case is therefore,
\[ \half u(\$400) + \half ( \half u(1000) + \half u(2000))
 = \half (5.994) + \half(\half (6.909)+\half (7.601)) = 6.624~.\]	
Because $6.624 < 686$, it follows that the optimal policy is to choose $A$ in both stages.
\end{example}

Example \ref{ex:3} suggest that the general case is much more complicated.
First, denote by $D$ the set of possible decisions (``rewards") from which the DM has to select on in each stage.
Next, denote by $R_k$ the set of all possible values of total reward received during stages $1,\ldots,k$.
For example, if the rewards are as in Example \ref{ex:3} and there are three stages, then
$R_2 = \{0,400,800,1000,1400, 2000\}$.
For every $d\in D$ and $r\in R_1$, denote by $p(r,d)$ the probability of cash reward $r$ when the DM chooses $d\in D$.

Denote by $r_k$ the actual total reward received by the DM during stages $1,\ldots,k$,
and for $k>1$, let $X_k = X_k(r_{k-1})$ denote a random variable whose value is equal to the total reward received during
stages $k,k+1,\ldots,n$, assuming the DM follows an optimal policy for the subproblem  
$P_{k}(r_{k-1})$ that  consists only of these stages and given $r_{k-1}$;
optimality means the utility  $U(X_k\,|\,r_{k-1})$ from the random-variable total reward received during those stages is maximized.  Also, let $X_1$ denote the optimal expected utility in the entire $n$-stage problem

The dynamic-programming equation that characterizes optimality is as follows.
For every $k>1$ and  $r_{k-1} \in R_{k-1}$
\[ X_k(r_{k-1}) =
\max_{d \in D} \bigg\{ \sum_{r\in R_1}  p_{rd}\cdot  \big[ U(r\,|\,r_{k-1}) + X_{k+1}(r_{k-1} + r) \big]\bigg\}~.
\]
Finally,
\[ X_1 =
\max_{d \in D} \bigg\{ \sum_{r\in R_1}  p_{rd}\cdot  \big[ u(r) + X_{2}(r) \big]\bigg\}~.
\]

\begin{center}
\bibliographystyle{plain}
\bibliography{repeated_bets}
 \end{center}
 
 \end{document}